\documentclass{article}
\usepackage{arxiv}
\usepackage{microtype}
\usepackage{graphicx}
\usepackage{subfigure}
\usepackage{booktabs} 
\usepackage{hyperref}
\usepackage{xcolor}
\hypersetup{
    colorlinks,
    linkcolor={red!50!black},
    citecolor={blue!50!black},
    urlcolor={blue!80!black}
}
\usepackage[round]{natbib}

\usepackage{amsmath}
\usepackage{amssymb}
\usepackage{mathtools}
\usepackage{amsthm}
\usepackage{subfigure}
\usepackage{smile}
\usepackage{xspace}
\usepackage[capitalize,noabbrev]{cleveref}

\theoremstyle{plain}
\newtheorem{theorem}{Theorem}[section]

\theoremstyle{definition}

\theoremstyle{remark}

\usepackage[textsize=tiny]{todonotes}

\newcommand{\smurf}{\texttt{SMURF-THP}\xspace}

\newcommand{\stitle}[1]{\vspace{1.5ex}\noindent{\bf #1}}
\newcommand{\etitlet}[1]{\vspace{0.0ex}\noindent{\underline{\em #1}}}
\mathchardef\mhyphen="2D

\begin{document}

\title{\smurf: Score Matching-based UnceRtainty quantiFication for Transformer Hawkes Process\thanks{Published as a conference paper in ICML 2023.}}

\date{Oct 24, 2023}

\author{Zichong Li\thanks{Li, Xu, Zuo, Zhang and Zhao are affiliated with Georgia Institute of Technology. Jiang is affiliated with Amazon. Zha is affiliated with The Chinese University of Hong Kong. Correspondence to \url{zli911@gatech.edu}, \url{tourzhao@gatech.edu} and \url{zhahy@cuhk.edu.cn}.}, Yanbo Xu, Simiao Zuo, Haoming Jiang, Chao Zhang, Tuo Zhao, Hongyuan Zha}

\renewcommand{\shorttitle}{\smurf: Score Matching-based UnceRtainty quantiFication for Transformer Hawkes Process}

\maketitle

\begin{abstract}
Transformer Hawkes process models have shown to be successful in modeling event sequence data. However, most of the existing training methods rely on maximizing the likelihood of event sequences, which involves calculating some intractable integral. Moreover, the existing methods fail to provide uncertainty quantification for model predictions, e.g., confidence intervals for the predicted event's arrival time. 
To address these issues, we propose \smurf, a score-based method for learning Transformer Hawkes process and quantifying prediction uncertainty. Specifically, \smurf learns the score function of events' arrival time based on a score-matching objective that avoids the intractable computation. With such a learned score function, we can sample arrival time of events from the predictive distribution. This naturally allows for the quantification of uncertainty by computing confidence intervals over the generated samples. We conduct extensive experiments in both event type prediction and uncertainty quantification of arrival time. In all the experiments, \smurf outperforms existing likelihood-based methods in confidence calibration while exhibiting comparable prediction accuracy.
\end{abstract}

\section{Introduction}
Sequences of discrete events in continuous time are routinely generated in many domains such as social media \citep{yang2011like}, financial transactions \citep{bacry2015hawkes}, and personalized treatments \citep{wang2018supervised}. A common question to ask is: Given the observations of the past events,
when and what type of event will happen next?
Temporal point process such as the Hawkes process \citep{hawkes1971spectra} uses an \textit{intensity function} to model the arrival time of events.
A conventional Hawkes process defines a parametric form of the intensity function, which assumes that the influence of past events on the current event decreases over time. Such a simplified assumption limits the model's expressivity and fails to capture complex dynamics in modern intricate datasets. Recently, neural Hawkes process models parameterize the intensity function using recurrent neural networks \citep{rmtpp2016, nhp2017} and Transformer \citep{thp2020,sahp2020}. These models have shown to be successful in capturing complicated event dependencies. 

However, there are two major drawbacks with existing training methods for neural Hawkes process models.
First, they mainly focus on maximizing the log-likelihood of the events' arrival time and types.
To calculate such likelihood, we need to compute the integral of the intensity function over time, which is usually computationally challenging, and therefore numerical approximations are often applied. As a result, not only is the computation complex, but the estimation of the intensity function can still be inaccurate in these likelihood-based methods.
Second, the arrival time of events often follows a long tail distribution, such that point estimates of the arrival time are unreliable and insufficient. Thus, confidence intervals that can quantify model uncertainty are necessary. However, calculating the confidence interval directly from density is computationally challenging. This is because we need to compute the coverage probability, which is essentially another integral of the intensity function (recall that the intensity is an infinitesimal probability).

To overcome the above two challenges, we propose \smurf, a score-based method for learning Transformer Hawkes process models and quantifying the uncertainty of the models' predicted arrival time.
The score-based objective in \smurf, originally introduced by \citet{hyvarinen2005estimation}, is to match the derivative of the log-likelihood (known as \textit{score}) of the observed events' arrival time to the derivative of the log-density of the empirical (unknown) distribution. Note that calculating scores no longer requires computing the intractable integral in the likelihood function. Moreover, \smurf is essentially a score-based generative model: With the neural score function learnt by \smurf, we can generate samples from the predicted distribution of the arrival time. This naturally allows for the quantification of uncertainty in predictions by computing confidence intervals over the generated samples.

Note that Transformer Hawkes process jointly models event types and arrival time. 
Because score matching methods can only be applied to handle the continuous arrival time but not the categorical event types, 
we need to decompose the joint likelihood into a marginal likelihood of the arrival time and a partial likelihood of the types. When training the Transformer Hawkes process models using \smurf, the objective becomes a weighted sum of the score-matching loss (which replaces the marginal likelihood of the arrival time) and the negative partial log-likelihood of event types. 

We demonstrate the effectiveness of \smurf for learning Transformer Hawkes process models on four real-world datasets, each of which contains sequences of multiple types of discrete events in continuous time. We decompose the joint likelihood by conditioning event type on the arrival time and consider a Transformer Hawkes process where one single intensity function is shared across all event types. Our experimental results show that compared with likelihood-based methods, \smurf demonstrates superior performance on uncertainty quantification of arrival time, while exhibiting comparable event-type prediction accuracy. For example, \smurf achieves within 0.3\% prediction accuracy compared with likelihood-based methods, while gaining 4-10\% confidence calibration in predictions of arrival times.

The rest of the paper is organized as follows: Section 2 briefly reviews the background, Section 3 describes our proposed method, Section 4 presents the experimental results, and Section 5 discusses the related work and draws a brief conclusion.
\section{Background}
We briefly review the neural Hawkes process models, the Transformer model and the score matching method.
\subsection{Neural Hawkes Process}
For simplicity, we review neural Hawkes process with only one event type. And we will introduce Hawkes process with multiple event types (i.e., marked point process) in the next section. 
Let $\mathcal{S}=\{t_i\}_{i=1}^{L}$ denote an event sequence with length $L$, where each event is observed at time $t_i \in [0,t_{\max}]$. 
We further denote the history up to time $t$ as $\mathcal{H}_{t} = \{t_j:t_j<t\}$, and the conditional intensity function as $\lambda(t\mid \mathcal{H}_t)$.
Then, we have the conditional probability of the next event proceeding $t_{i-1}$ as:
\begin{align}
p_T(t\mid \mathcal{H}_{t_i}) = \lambda(t \mid \mathcal{H}_{t_i}) e^{-\int_{t_{i-1}}^t \lambda(\tau\mid \mathcal{H}_{t_i}) d \tau} .
\label{density}
\end{align}
The log-likelihood of the event sequence $\mathcal{S}$ is then
\begin{align}
    \ell_{0}(\mathcal{S}) = \sum_{i=1}^L \log \lambda(t_i | \mathcal{H}_{t_i}) - \int_0^{t_{\max}} \lambda(\tau | \mathcal{H}_{t_i}) d \tau.
\label{loglikelihood}
\end{align}
Notice that exact computation of the integral in Eq.~\eqref{density} is intractable, such that numerical approximations are required.

\textit{Neural Hawkes Process} parameterizes the intensity function $\lambda$ with deep neural networks.
For example, \citet{rmtpp2016, nhp2017} use variants of recurrent neural networks, and \citet{thp2020, sahp2020} use the Transformer model \citep{attention}. Other examples include employments of deep Fourier kernel \citep{SAPP} and feed-forward neural networks \citep{FNNint}.

\subsection{Transformer Model}

In this paper, we parameterize the intensity function $\lambda$ in Eq.~\eqref{density} using a Transformer model \citep{attention}. The Transformer model has multiple Transformer layers, where each layer contains a self-attention mechanism and a feed-forward neural network. 
The self-attention mechanism assigns attention weights to every pair of events in a sequence. These weights signify the strength of the dependency between pairs of events, with smaller weights indicating weaker dependencies and larger weights suggesting stronger ones.
Such mechanism efficiently models the dependencies between events irrespective of their position in the sequence, thereby capturing long-term effects.
The feed-forward neural network in each layer further incorporates non-linearities to offer a larger model capacity for learning complex patterns.
Studies have shown that the Transformer model outperforms recurrent neural networks in modeling event sequences \citep{thp2020, sahp2020}.

\subsection{Score Matching}
\textit{Score matching} \citep{score_matching2005} is originally designed for estimating non-normalized statistical methods on $\mathbb{R}^n$ without challenging computation of the normalization term.
In practice, however, score matching is not scalable to high-dimensional data and deep neural networks due to the calculation of derivative of density.
Denoising score matching \citep{denoise2011} is one of its variants that circumvent the derivative by adding perturbation to data points.
\citet{slice-sm} also enhance the scalability by projecting the score function onto random vectors.
\citet{gsm2019} and \citet{gsm2022} further generalize the origin score matching to accommodate densities supported on a more general class of domains.
Early work \citep{sahani2016score} has adopted score matching in modeling Poisson point process, but it fails to model complicated event dependencies induced in the modern event data due to the simplified assumption that the intensity function is independent of the historical events.
A suitable sampling algorithm for score matching-based models is \textit{Langevin Dynamics (LD)}, which can produce samples from a probability density using only its score function.
\citet{mirror} propose Mirror Langevin Dynamics (MLD) as a variant of Langevin Dynamics that focuses on sampling from a constrained domain.
We employs both LD and MLD for generating event samples in our experiments.

\section{Method}
\label{sec:method}
\subsection{Score Matching Objective of Hawkes Process}
Let $\mathcal{S}=\{(t_i,k_i)\}_{i=1}^{L}$ denote an event sequence of length $L$, where each pair $(t_i,k_i)$ corresponds to an event of type $k_i \in \{1, \cdots, M\}$ happened at time $t_i \in [0,t_{\mathrm{max}}]$.
Also, we denote the history events up to time $t$ as $\mathcal{H}_{t} = \{(t_j,k_j):t_j<t\}$.
Our goal is to learn $p_{T,K}(t,k \given \mathcal{H}_{t})$, the joint conditional probability of the event proceeding time $t$ given the history $\mathcal{H}_t$.

To employ score matching, we decompose the joint conditional pdf $p_{T,K}$ by conditioning on the event time.
By doing so, the partial likelihood of the discrete event types can be maximized by minimizing the cross entropy loss; and the marginal likelihood of the continuous event time can be substituted by a score-matching objective. Such a substitution avoids the intractable integral of the intensity.
Specifically, we condition $p_{T,K}$ on the event time and we have:
\begin{align}
p_{{T,K}}(t,k \given \mathcal{H}_{t}) = p_T(t\given \mathcal{H}_{t}) \cdot p_{{K\given T}}(k\given t,\mathcal{H}_{t}).
\end{align}
Correspondingly, we have the log-likelihood:
\begin{align}
    \ell(\mathcal{S})\!=\!
    \sum_{i=1}^L \log p_T(t_i | \mathcal{H}_{t_i})
    \!+\! \sum_{i=1}^L \log p_{{K | T}}(k_i | t_i,\mathcal{H}_{t_i}).
    \label{ll-1}
\end{align}
We can use a neural network to parameterize the intensity function and directly train the model using Eq.~\eqref{ll-1}. However, such an approach inevitably faces computational challenges, i.e., exact computation of the intensity's integral is intractable. Therefore, we derive a score-matching objective to substitute the first term in Eq.~\eqref{ll-1}.

In \smurf, we use a Transformer model with parameters $\theta$ to parameterize the intensity function.
More details are presented in Section \ref{sec:parametrize}.
A sample's \textit{score} is defined as the gradient of its log-density.
Then, using Eq.~\eqref{density}, we can write the score of the $i$-th event given its history $\mathcal{H}_{t_i}$ and the model parameters $\theta$ as:
\begin{align}
{\psi}(t_i\given\mathcal{H}_{t_{i}};\theta) =\partial_{t} \log{p_{T}(t_i\given\mathcal{H}_{t_{i}};\theta)} \nonumber =\partial_{t} \log{\lambda(t_i\given\mathcal{H}_{t_{i}};\theta)} - \lambda(t_i\given\mathcal{H}_{t_{i}};\theta).
\end{align}
The original objective of score matching is to minimize the expected squared distance between the score of the model $\psi(\cdot;\theta)$ and the score of the ground truth $\psi^*(\cdot)$. However, minimizing such an objective is infeasible since it relies on the unknown score $\psi^*(\cdot)$. We can resolve this issue by following the general derivation in \citet{score_matching2005} and arrive at an empirical score-matching objective for Hawkes process with single type:
\begin{align}
\hat{J}(\theta)   = \sum_{i=1}^{L}[\frac{1}{2}{\psi}(t_{i}\given\mathcal{H}_{t_{{i}}};\theta)^2  + \partial_{t}{\psi}(t_{i}\given\mathcal{H}_{t_{{i}}};\theta)],
\label{eq:obj_ours1}
\end{align}
where $\partial_{t} \psi(t_i \given \mathcal{H}_{t_i};\theta) = \partial^2_{t} \log{\lambda(t_i\given \mathcal{H}_{t_{i}};\theta)} - \partial_{t} \lambda(t_i\given \allowbreak \mathcal{H}_{t_{i}};\theta)$.
We state in the follow theorem that the score matching objective in Eq. (\ref{eq:obj_ours1}) satisfies local consistency: minimizing $\hat{J}(\theta)$ is as sufficient as maximizing the first term of Eq. (\ref{ll-1}) for estimating the model parameters. 
 
\begin{theorem}
\label{theorem1}
Assume the event time in sequence $\mathcal{S}$ follows the model: $p_T^*(t\given \mathcal{H}_t)=p_T(t\given \mathcal{H}_{t};\theta^*)$ for some $\theta^*$, and that no other parameter gives a pdf that is equal \footnote{Equality of pdf's are taken in the sense of equal almost everywhere with respect to the Lebesgue measure.} to $p_T(\cdot;\theta^*)$.
Assume further that the optimization algorithm is able to find the global minimum and $p_T(t\given \mathcal{H}_{t};\theta)$ is positive for all $t$ and $\theta$.
Then the score matching estimator obtained by minimizing Eq. (\ref{eq:obj_ours1}) is consistent, i.e., it converges in probability towards $\theta^*$ when sample size approaches infinity.
\end{theorem}
\begin{proof}
Let $\psi^{*}(\cdot)$ and $\psi(\cdot;\theta)$ be the associated score function of $p_T^*(\cdot)$ and $p_T(\cdot;\theta)$, respectively.
The objective in Eq. (\ref{eq:obj_ours1}) is an empirical estimator of the following objective:
\begin{align}
    J(\theta)\!=\!\frac{1}{2}{\left\langle \sum_{i=1}^{
    L}  ( \psi(t_i \mid \mathcal{H}_{t_i};\theta) 
    - \psi^{*}(t_i \mid \mathcal{H}_{t_i}) )^2  \right\rangle}_{\{t_i\}_{i=1}^L \sim p_T^*}.
\end{align}
We first prove that $J(\theta)=0 \Rightarrow \theta=\theta^*$.
Since $p_T^*(\cdot)$ is positive, we can infer from $J(\theta')=0$ that $\psi(\cdot;\theta')$ and $\psi^{*}(\cdot)$ are equal, which implies $\log p_T^*(\cdot)=\log p_T(\cdot;\theta ')+c$ for some constant $c$.
Because both $p_T^*$ and $p_T$ are pdf's, the constant $c$ must be 0 and hence we have $p_T^*=p_T$.
By assumption, $\theta^*$ is the only parameter that fulfills this equation, so necessarily $\theta'=\theta^*$.

Then according to the law of large numbers, $\hat{J}$ converges to $J$ as the sample size approaches infinity.
Thus, the estimator converges to a point where $J$ is globally minimized.
Considering $J(\theta)=0 \Rightarrow \theta=\theta^*$, the minimum is unique and must be found at the true parameter $\theta^*$.
\end{proof}

We highlight that by substituting the first term of Eq.~\eqref{ll-1} with the score-matching objective in Eq.~\eqref{eq:obj_ours1}, we no longer need to compute the intractable integral.
For the second term in Eq.~\eqref{ll-1}, we employ a neural network classifier $g(t_i,\mathcal{H}_{t_i};\omega)$ parameterized by $\omega$. The classifier takes the timestamp $t_i$ and its history $\mathcal{H}_{t_i}$ as inputs, and the classifier outputs an event type distribution of the next event proceeding time $t_i$.
We can learn the model parameters $[\theta; \omega]$ by minimizing the following loss function
\begin{align}
\mathcal{L}(\mathcal{S}) = \alpha \hat{J}(\theta) + \sum_{i=1}^L \mathrm{CE}(g(t_i,\mathcal{H}_{t_i};\omega),k_i),
\end{align}
where $\alpha$ is a hyperparameter that controls the weight of the score matching objective and $\mathrm{CE}$ is the cross entropy loss.

\subsection{Training: Denoising Score Matching}
In practice, the direct score-based approach has limited success because of two reasons. First, the estimated score functions are inaccurate in low-density regions \citep{SMLD}. Second, the derivatives of the score functions in Eq. (\ref{eq:obj_ours1}) contain second-order derivatives, leading to numerical instability.

To alleviate these issues, we adopt denoising score matching \citep{denoise2011}, which perturbs the data with a pre-specified noise distribution.
Specifically, we add a Gaussian noise $\mathcal{N}(0,\sigma)$ to each observed event time $t_i$, obtaining a noise-perturbed distribution $q_{\sigma}(t\given t_i)$, the intent of which is to augment samples in the low density regions of the original data distribution.
Moreover, the denoising score matching objective circumvents computing second-order derivatives by replacing them with the scores of the noise distributions.
Suppose for the $i$-th event, we have $S$ perturbed event time denoted by $\{t_i^{\sigma,j}\}_{j=1}^S$, then we can substitute the objective $\hat{J}(\theta)$ in Eq. (\ref{eq:obj_ours1}) by
\begin{align}
\hat{J}^{\sigma}(\theta) \!=\! \sum_{i=1}^{L}\sum_{j=1}^{S} \frac{1}{2} \Vert {\psi}(t_{i}^{\sigma,j}\given\mathcal{H}_{t_{i}};\theta)^2  - \partial_{t}\log q_{\sigma}(t_{i}^{\sigma,j}\given t_i) \Vert ^2.
\end{align}
Here, $\partial_{t}\log q_{\sigma}(t\given t_i)=\frac{t_i-t}{\sigma^2}$ is the score of the noise distribution, which directs the perturbed timestamps to move towards the original time $t_i$.

\subsection{Parametrization}
\label{sec:parametrize}
We now introduce how we construct our model. We follow Transformer Hawkes Process \citep{thp2020, sahp2020} that leverages self-attention mechanism to extract features from sequence data.
Given input sequence $\mathcal{S}$, each event are firstly encoded by the summation of temporal encoding and event type embedding \citep{thp2020}.
Then we pass the encoded sequence $\mathbf{C}\in \mathbb{R}^{L\times D}$ through the self-attention module that compute the multi-head attention output $\mathbf{A}$ as 
\begin{align*}
&\mathbf{A}=\mathrm{Concat}(\mathbf{A}_1,...,\mathbf{A}_H)\mathbf{W}^{O},\\
&\mathbf{A}_h={\mathrm{Softmax}}(\frac{(\mathbf{CW}_h^{Q})^{\top}\mathbf{CW}_h^{K}}{\sqrt{d_k}})\mathbf{CW}_h^{V}.
\end{align*}
Matrices $\mathbf{W}^{Q}_h,\mathbf{W}^{K}_h \in \mathbb{R}^{d\times d_k}$ and $\mathbf{W}^{V}_h \in \mathbb{R}^{d\times d_v}$ are query, key, value projections, and
$\mathbf{W}^{O}\in\mathbb{R}^{(d_v*H)\times d}$ aggregates the final attention outputs.
Then $\mathbf{A}$ is fed through a position-wise feed-forward neural network (FFN) to obtain the hidden representations as
\begin{align*}
\mathrm{FFN}(\mathbf{A}) &= \max(0,\mathbf{AW}_1^{\mathrm{FC}}+\mathbf{b_1})\mathbf{W}_2^{\mathrm{FC}} +\mathbf{b_2},
\end{align*}
 where the $j^{th}$ row of $\mathrm{FFN}(\mathbf{A})$ encodes the $j^{th}$ event and all past events up to time $t_j$.
In practice, we stack multiple self-attention modules and position-wise feed-forward networks to construct a model with a larger capacity.
We add future masks while computing the attention to avoid encoding future events.

After generating hidden representations for event sequence, we parametrize the total intensity function $\lambda$ as $f$ and the conditional distribution of next event type $p_{K\given T}$ as $g$.
The parametric functions are defined by
\begin{align*}
f(t,\mathcal{H}_{t_{j+1}}) & = \mathrm{Tanh}(\mathbf{h}(j)[ \mathbf{W}_1^f (t-t_j)+\mathbf{W}_2^f ]+\mathbf{b}_1^f){\mathbf{w}_3^f}^\top \quad + \mathbf{b}_2^f,\\
g(t, \mathcal{H}_{t_{j+1}}) & = {\mathrm{Softmax}}(\mathbf{h}(j)\left [ \mathbf{W}_1^g (t-t_j)+ \mathbf{W}_2^g \right ]+\mathbf{b}^g),
\end{align*}
where $\mathbf{W}_1^f,\mathbf{W}_2^f,\mathbf{w}_3^f,\mathbf{W}_1^g,\mathbf{W}_2^g$ are trainable weight matrix.

\subsection{Uncertainty Quantification}
\label{sampling}
Using the learnt score function, we can generate new events using the Langevin Dynamics (LD) and compute confidence intervals.
Without loss of generality, we denote the learnt score function as $\psi(t \given \mathcal{H}_t)=\partial_t \log p(t \given \mathcal{H}_t)$. Suppose we aim to generate a sample for the $i^{th}$ event, we first generate an initial time gap $t^{(0)}\sim \pi(t)$ with $\pi$ being a uniform prior distribution.
Then, the LD method recursively updates the time gap given history $\mathcal{H}_{t_i}$ by
\begin{align}
    t^{(n)} = t^{(n-1)} + \frac{\epsilon}{2}\psi(t_{i-1}+t^{(n-1)} \given \mathcal{H}_{t_{i}}) +\sqrt{\epsilon} z_n.
\end{align}
Here, $\epsilon$ is a step size and we have $z_n \sim \mathcal{N}(0,1)$ for $n=1, \cdots, N$.
The distribution of $t^{(N)}$ is proved to approximate the predicted distribution as $\epsilon \to 0$ and $N \to \infty$ under certain regularity conditions \citep{langevin_cov}.
Note that negative time gap can appear in both the perturbation process and the sampling process due to the added random Gaussian noise. This violates the constraint that an event's arrival time should be non-decreasing.
To resolve this issue, we adopt Tweedie's formula \citep{tweedie} and add an extra denoising step after the LD procedure:
\begin{align}
\hat{t} = t^{(N)}+\sigma \psi(t_{i-1}+t^{(N)} \given \mathcal{H}_{t_{i}}).
\end{align}
Here, $\sigma$ is the standard deviation of the Gaussian noises. The final event time sample is then $t_{i-1}+\hat{t}$.

Using the above sampling algorithm, we can generate event time samples under different histories for uncertainty quantification. As for event type, we sample from the learnt conditional density function $g(t_{i-1}+\hat{t}, \mathcal{H}_{t_i}; \hat{\omega})$ for each time sample.

\section{Experiments}
\setlength\abovedisplayskip{8pt}
\setlength\belowdisplayskip{8pt}
We demonstrate the effectiveness of our method on four real-world datasets with four baselines. Our code is publicly available at \href{https://github.com/zichongli5/SMURF-THP}{https://github.com/zichongli5/SMURF-THP}.

\subsection{Setup}
\stitle{Datasets.} We experiment on four real-world datasets. Table~\ref{dataset} summarizes their statistics.
\begin{table}[thb!]
\centering
\caption{Datasets statistics. From left to right columns: name of the dataset, number of event types, number of events, and average length per sequence.}
\vskip 0.15in
\label{dataset}
\begin{tabular}{cccc}
\toprule
Dataset       & \#Type & \#Event  & Avg. Length \\ \midrule
StackOverflow & 22    & 480413  & 64             \\
Retweet       & 3     & 2173533 & 109            \\
MIMIC-II      & 75    & 2419    & 4              \\
Financial     & 2     & 414800  & 2074           \\ \bottomrule
\end{tabular}
\end{table}

\etitlet{StackOverflow \citep{so_data}}
StackOverflow is a website that serves as a platform for users to ask and answer questions.
Users will be awarded badges based on their proposed questions and their answers to others.
This dataset contains sequences of $6,633$ users' reward history in a two-year period, where each event signifies receipt of a particular type of medal.

\etitlet{Retweet \citep{retweet_data}}
This dataset contains $24,000$ sequences of tweets.
Each sequence starts with an original tweet at time 0, and the following events signify retweeting by other users.
All users are grouped into three categories based on the number of their followers.
So each event is labeled with the retweet time and the retweeter's group.

\etitlet{MIMIC-II \citep{rmtpp2016}}
This dataset contains a subset of $650$ patient's clinical visits to the Intensive Care Units in a seven-year period.
Each sequence consist of visits from one particular patient, and each event is labeled with a timestamp and a set of diagnosis codes.

\etitlet{Financial Transactions \citep{rmtpp2016}}
This dataset contains a total of $0.7$ million transaction records for a stock from the New York Stock Exchange. The long single sequence of transactions are partitioned into $2,000$ subsequences, where each event is labeled with the transaction time and the action that was taken: buy or sell.

For all the aforementioned datasets, we adhere to the same data pre-processing and train-dev-test partitioning as described in \citet{rmtpp2016} and \citet{nhp2017}. Additionally, we apply normalization and log-normalization techniques to enhance performance. In particular, we normalize the event time in the MIMIC-II dataset, while for the Retweet and Financial datasets, we perform log-normalization of time using the formula $\frac{\log(t)-\mathrm{Mean}(\log(t))}{\mathrm{Var}(log(t))}$. During the testing phase, we rescale the generated events' timestamps to their original scale and assess their quality.

\stitle{Baselines.} We compare our method to four existing works:

\noindent $\bullet$ Neural Hawkes Process (NHP, \citet{nhp2017}), which designs a continuous-time LSTM to model the evolution of the intensity function between events; 

\noindent $\bullet$ Noise-Contrastive Estimation for Temporal Point Process (NCE-TPP, \citet{ncetpp2020}), which adopts the noise-contrastive estimator to bypass the computation of the intractable normalization constant;

\noindent $\bullet$ Self-Attentive Hawkes Process (SAHP, \citet{sahp2020}), which proposes a time-shifted positional encoding and adopts self-attention to model the intensity function;

\noindent $\bullet$ Transformer Hawkes Process (THP, \citet{thp2020}), which adopts the Transformer architecture to capture long-term dependencies in history and designs a continuous formulation for the intensity function.

For a fair comparison, we employ the same backbone architecture as THP and use default hyperparameters for NHP and SAHP. We implement NCE-TPP using the Transformer backbone in place of its original LSTM structure.
Because these methods do not support uncertainty quantification originally, we calculate the scores from their estimated intensity functions that were learnt through maximum likelihood or noise-contrastive estimation. Then we can follow the similar LD sampling procedure of ours to generate samples using these baseline models.

\stitle{Metrics.} We apply five metrics for measuring the quality of the generated event samples:

\noindent $\bullet$ \textit{Calibration Score} calculates the Root Mean Squared Error (RMSE) between the coverage of the confidence intervals produced by the samples and that of the desired probabilities. 
Specifically, with $U$ generated samples $\{{\hat{t}_i}^{~j}\}_{j=1}^U$ for each time point $t_i$, we can estimate a $q_s$-confidence interval as $[0,t_i^{q_s}]$, where $t_i^{q_s}$ is the $q_s$-quantile of the $U$ samples.
Then we can compute the coverage of the estimated $q_s$-intervals by counting how many times the true value $t_i$ falls in the estimated intervals, i.e., $c_s = \sum_{i=1}^L \mathbb{I}_{t_i<t_i^{q_s}}$.
Calibration Score is defined as the average RMSE between the coverage and the confidence level $q_s$.
We consider to compute CS scores at quantiles of $\{0.5,0.55,...,0.95\}$ because the samples generated below $0.5$ quantiles are less useful and can be still noisy even using the denoising sampling method;

\noindent $\bullet$ \textit{Continuous Ranked Probability Score (CRPS)} measures the compatibility of the estimated cumulative distribution function (cdf) $\hat{F}$ per time $t_i$ as 
\begin{align*}
\mathrm{CRPS}(\hat{F},t_i) = \int_{0}^\infty(\hat{F}(t) - \mathbb{I}_{t\geq t_i})^2 dt.
\end{align*}
We compute the empirical CRPS \citep{crps} for time $t_i$ from the samples by
\begin{align*}
\mathrm{CRPS}=\frac{1}{LU}\sum_{i=1}^L[\sum_{j=1}^U|\hat{t}_i^{~j}-t_i|-\frac{1}{2U}\sum_{j=1}^U\sum_{k=1}^U|\hat{t}_i^{~j}-\hat{t}_i^{~k}|].
\end{align*}
The total CRPS is then defined as the average among the empirical CRPS of all events.

\noindent $\bullet$ \textit{Interval Length (IL)} computes the average length of the predicted intervals for a particular confidence level $q_s$ as 
\setlength{\abovedisplayskip}{6pt}
\setlength{\belowdisplayskip}{6pt}
\begin{align*}
\mathrm{IL}=\frac{1}{L}\sum_{i=1}^L t_i^{q_s}.
\end{align*}

\noindent $\bullet$ \textit{Coverage Error (CER)} calculates the mean absolute error between a coverage $c_s$ and the corresponding confidence level $q_s$ as
\begin{align*}
\mathrm{CER} = |c_s-q_s|.
\end{align*}

We report IL and CER at $q_s=0.5$ in the main results. For all of the above metrics, lower is considered ``better''.


\noindent $\bullet$ \textit{Type Prediction Accuracy} measures the quality of the generated event's types. We take the mode of the sampled types as the prediction and then compute the accuracy.

\subsection{Results}
\begin{table*}[tb]
\centering
\caption{Comparison of different methods' performance on four real-world datasets in terms of Calibration Score (CS), CRPS, Coverage Error (CER) and Interval Length (IL). CER and IL are both calculated at confidence level $0.5$.}
\vskip 0.15in
\label{main_results}
\resizebox{\textwidth}{!}{%
\begin{tabular}{lcccccccccc}
\toprule
                                  & \multicolumn{5}{c}{StackOverflow}                                                                                                                                 & \multicolumn{5}{c}{Retweet}                                                                                                                                       \\ \cmidrule(r){2-6} \cmidrule(r){7-11} 
Methods                           & $\mathrm{CS}(\%)(\downarrow)$  & $\mathrm{CRPS}(\downarrow)$    & $\mathrm{CER}(\%)(\downarrow)$  & $\mathrm{IL}(\downarrow)$   & $\mathrm{Acc}(\%)(\uparrow)$ & $\mathrm{CS}(\%)(\downarrow)$  & $\mathrm{CRPS}(\downarrow)$   & $\mathrm{CER}(\%)(\downarrow)$  & $\mathrm{IL}(\downarrow)$    & $\mathrm{Acc}(\%)(\uparrow)$ \\ \midrule
NHP                   & $6.64\scriptstyle{\pm 0.55}$ & $0.69\scriptstyle{\pm 0.01}$ & $5.78\scriptstyle{\pm 0.26}$ & $0.67\scriptstyle{\pm 0.04}$   & $46.20\scriptstyle{\pm 0.13}$   & $14.35\scriptstyle{\pm 0.22}$ & $1.65\scriptstyle{\pm 0.05}$ & $16.98\scriptstyle{\pm 0.24}$ & $0.077\scriptstyle{\pm 0.005}$    & $60.28\scriptstyle{\pm 0.08}$                      \\
NCE-TPP            & $5.64\scriptstyle{\pm 0.46}$ & $0.62\scriptstyle{\pm 0.01}$ & $5.99\scriptstyle{\pm 0.39}$ & $0.63\scriptstyle{\pm 0.04}$   & $46.20\scriptstyle{\pm 0.15}$   & $13.55\scriptstyle{\pm 0.10}$ & $1.59\scriptstyle{\pm 0.05}$ & $10.30\scriptstyle{\pm 0.16}$ & $0.055\scriptstyle{\pm 0.002}$    & $60.30\scriptstyle{\pm 0.06}$                      \\
SAHP                   & $4.48\scriptstyle{\pm 0.24}$ & $0.48\scriptstyle{\pm 0.02}$ & $7.73\scriptstyle{\pm 0.47}$ & $0.55\scriptstyle{\pm 0.01}$   & $46.22\scriptstyle{\pm 0.06}$  & $10.06\scriptstyle{\pm 0.35}$ & $1.12\scriptstyle{\pm 0.02}$ & $15.65\scriptstyle{\pm 0.03}$ & $0.061\scriptstyle{\pm 0.001}$    & $60.32\scriptstyle{\pm 0.08}$                      \\
THP                    & $4.13\scriptstyle{\pm 0.25}$ & $0.46\scriptstyle{\pm 0.01}$ & $5.90\scriptstyle{\pm 0.33}$ & $0.56\scriptstyle{\pm 0.02}$  & $\mathbf{46.48}\scriptstyle{\pm 0.05}$& $4.12\scriptstyle{\pm 0.13}$ & $1.08\scriptstyle{\pm 0.05}$ & $3.00\scriptstyle{\pm 0.02}$ & $0.059\scriptstyle{\pm 0.002}$    & $\mathbf{60.63}\scriptstyle{\pm 0.13}$                      \\
$\smurf$ & $\mathbf{0.65}\scriptstyle{\pm 0.12}$& $\mathbf{0.44}\scriptstyle{\pm 0.01}$& $\mathbf{0.48}\scriptstyle{\pm 0.09}$& $\mathbf{0.52}\scriptstyle{\pm 0.01}$& $46.26\scriptstyle{\pm 0.08}$        & $\mathbf{0.71}\scriptstyle{\pm 0.16}$  & $\mathbf{0.86}\scriptstyle{\pm 0.04}$ & $\mathbf{0.76}\scriptstyle{\pm 0.03}$  & $\mathbf{0.031}\scriptstyle{\pm 0.001}$ & 
$60.34\scriptstyle{\pm 0.12}$  \\
\midrule

                                  & \multicolumn{5}{c}{MIMIC-II}                                                                                                                                   & \multicolumn{5}{c}{Financial}                                                                                                                                   \\ \cmidrule(r){2-6} \cmidrule(r){7-11}  
Methods                           & $\mathrm{CS}(\%)(\downarrow)$  & $\mathrm{CRPS}(\downarrow)$   & $\mathrm{CER}(\%)(\downarrow)$  & $\mathrm{IL}(\downarrow)$ & $\mathrm{Acc}(\%)(\uparrow)$ & $\mathrm{CS}(\%)(\downarrow)$  & $\mathrm{CRPS}(\downarrow)$   & $\mathrm{CER}(\%)(\downarrow)$  & $\mathrm{IL}(\downarrow)$  & $\mathrm{Acc}(\%)(\uparrow)$ \\ \midrule
NHP                    & $9.87\scriptstyle{\pm 0.35}$ & $0.74\scriptstyle{\pm 0.03}$ & $8.85\scriptstyle{\pm 0.29}$ & $0.69\scriptstyle{\pm 0.04}$ &    $83.25\scriptstyle{\pm 0.27}$ & $4.64\scriptstyle{\pm 0.42}$ & $1.03\scriptstyle{\pm 0.01}$ & $3.85\scriptstyle{\pm 0.30}$ & $0.053\scriptstyle{\pm 0.002}$  & $60.23\scriptstyle{\pm 0.05}$                      \\
NCE-TPP                    & $6.52\scriptstyle{\pm 0.24}$ & $0.65\scriptstyle{\pm 0.01}$ & $9.23\scriptstyle{\pm 0.38}$ & $0.77\scriptstyle{\pm 0.04}$   & $83.60\scriptstyle{\pm 0.32}$       & $3.55\scriptstyle{\pm 0.12}$ & $1.35\scriptstyle{\pm 0.05}$ & $4.09\scriptstyle{\pm 0.04}$ & $0.055\scriptstyle{\pm 0.04}$    & $60.28\scriptstyle{\pm 0.09}$                      \\
SAHP                   & $7.49\scriptstyle{\pm 0.32}$ & $0.61\scriptstyle{\pm 0.01}$ & $15.70\scriptstyle{\pm 0.52}$ & $\mathbf{0.35}\scriptstyle{\pm 0.01}$ & 
$83.76\scriptstyle{\pm 0.13}$        & $3.50\scriptstyle{\pm 0.33}$ & $1.09\scriptstyle{\pm 0.05}$ & $3.19\scriptstyle{\pm 0.06}$ & $0.048\scriptstyle{\pm 0.003}$  & $60.37\scriptstyle{\pm 0.08}$                      \\
THP                    & $3.89\scriptstyle{\pm 0.14}$ & $0.75\scriptstyle{\pm 0.05}$ & $5.81\scriptstyle{\pm 0.12}$ & $0.84\scriptstyle{\pm 0.03}$ & 
$\mathbf{84.78}\scriptstyle{\pm 0.13}$ & $3.53\scriptstyle{\pm 0.19}$ & $1.57\scriptstyle{\pm 0.02}$ & $3.19\scriptstyle{\pm 0.13}$ & $0.042\scriptstyle{\pm 0.006}$  & $60.51\scriptstyle{\pm 0.06}$                      \\
$\smurf$ & $\mathbf{2.87}\scriptstyle{\pm 0.11}$  & $\mathbf{0.55}\scriptstyle{\pm 0.04}$ & $\mathbf{2.33}\scriptstyle{\pm 0.09}$  & $0.43\scriptstyle{\pm 0.03}$ & $84.02\scriptstyle{\pm 0.31}$  & $\mathbf{2.49}\scriptstyle{\pm 0.13}$  & $\mathbf{0.84}\scriptstyle{\pm 0.01}$ & $\mathbf{2.61}\scriptstyle{\pm 0.06}$  & $\mathbf{0.042}\scriptstyle{\pm 0.003}$ & $\mathbf{61.02}\scriptstyle{\pm 0.09}$  \\  
\bottomrule
\end{tabular}%
}
\end{table*}

\textbf{Comparison of Uncertainty Quantification.}
We compare different methods' performance in Table~\ref{main_results}.
We can see that $\smurf$ outperforms other baselines by large margins in terms of CS and CER.
It also achieves the lowest CRPS and IL at confidence level of $0.5$ for most datasets.
This indicates that $\smurf$ provides the most precise and concurrently, the narrowest confidence intervals.
The improvements come from two reasons: 1) \smurf can generate samples in higher quality using the scores that are more accurately learned by minimizing the score matching objective. In contrast, the baselines do not initially support sampling and their scores are less accurate as they are derived from the intensity functions that are estimated by maximizing an approximated log-likelihood.
2) \smurf can utilize the denoising score matching method to perturb the data points and effectively reduce the area of low-density regions in the original data space, whereas other baselines do not support this.

\textbf{Comparison of Event Type Prediction.}
We also compare the event type prediction accuracy of all the methods in Table~\ref{main_results}.
NHP, NCE-TPP and SAHP predict event type by the estimated expectation of intensity, while THP and our method adopt an additional classifier for prediction.
As depicted in the table, $\smurf$ attains slightly lower accuracy compared to THP, but surpasses other baselines.
We attribute the accuracy loss to the fact that $\smurf$ is trained to predict event type given a ground truth timestamp but is tested on a sampled timestamp since it is unknown in the testing.
As elaborated in Section \ref{gt-acc-sec}, $\smurf$ can achieve comparable or even superior accuracy when tested using true timestamps.

\textbf{Different Confidence Level.}
Figure~\ref{coverage} and Figure~\ref{intlen} compare the coverage and the interval length of the confidence intervals at different levels across four datasets.
The black line in Figure~\ref{coverage} signifies the desired coverage.
Compared with THP and SAHP, our method is much closer to the black line, which indicates that the generated confidence intervals are more accurate.
Notice that the coverage error on small probabilities is larger than that on large probabilities.
The reason is that sampling around 0 is more challenging as the Langevin step size is relatively larger and the score function changes more rapidly.
$\smurf$ effectively mitigates such bias by learning a more precise event distribution.
Figure~\ref{intlen} reveals that \smurf also achieves a shorter interval length than THP and SAHP for most of the assessed confidence levels.
Despite $\smurf$ experiencing longer intervals on small probability for the StackOverflow dataset, the corresponding coverage of the baselines is less precise.
The improvement on the Retweet dataset is particularly notable since its distribution of event time exhibits a longer tail, where adding perturbation is necessary for modeling scores on low-density regions.
\begin{figure}[thb!]
\centering
\subfigure[StackOverflow]{   
    \includegraphics[width=0.24\linewidth]{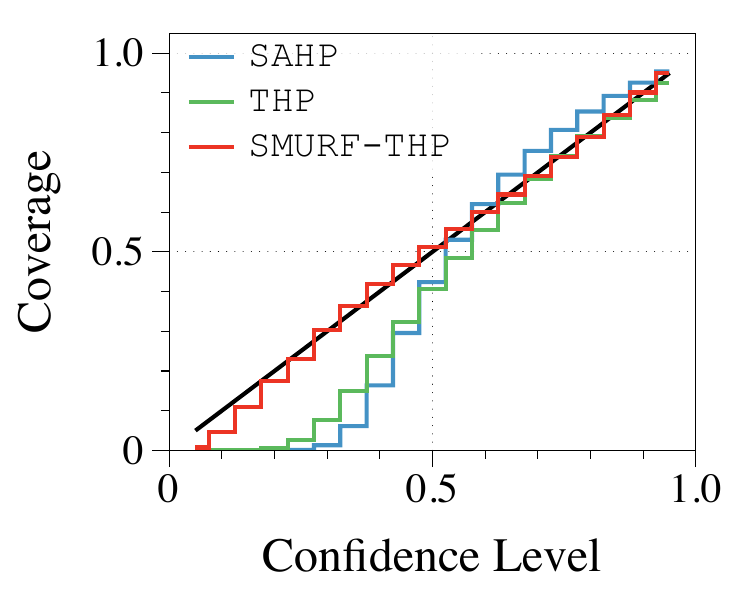}%
}%
\subfigure[Retweet]{   
    \includegraphics[width=0.24\linewidth]{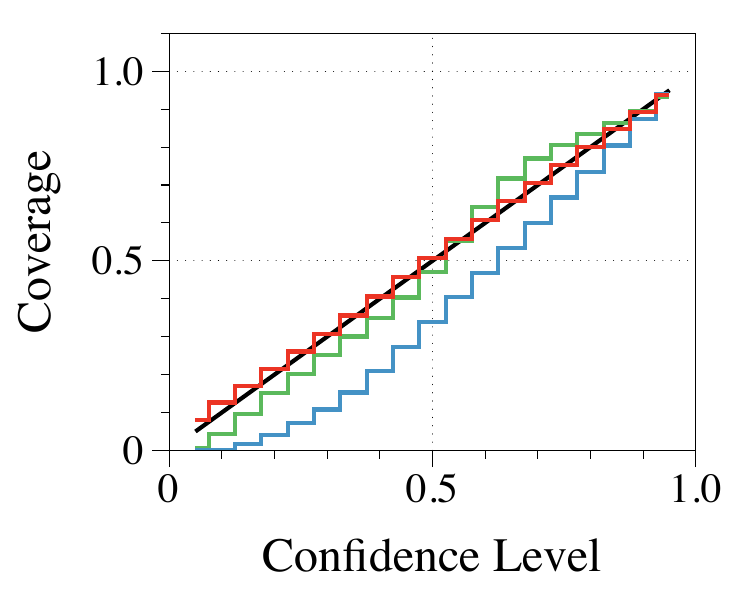}%
}%
\subfigure[MIMIC-II]{   
    \includegraphics[width=0.24\linewidth]{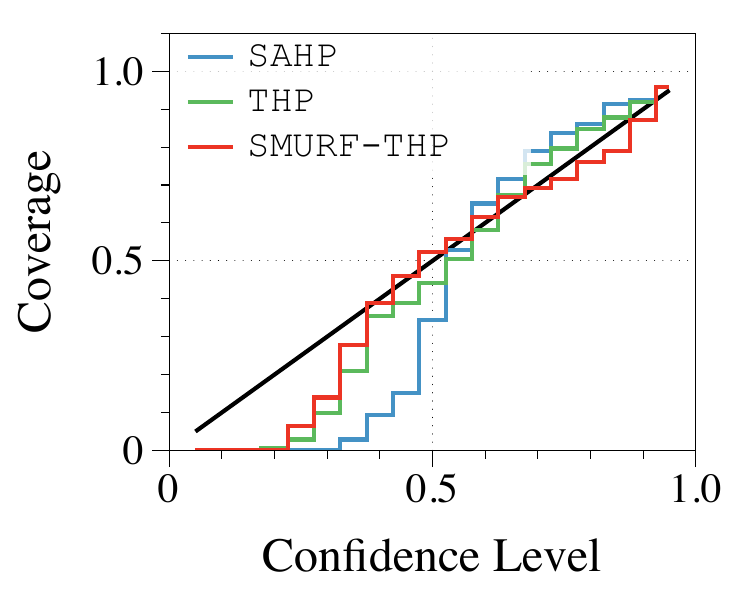}%
}%
\subfigure[Financial]{   
    \includegraphics[width=0.24\linewidth]{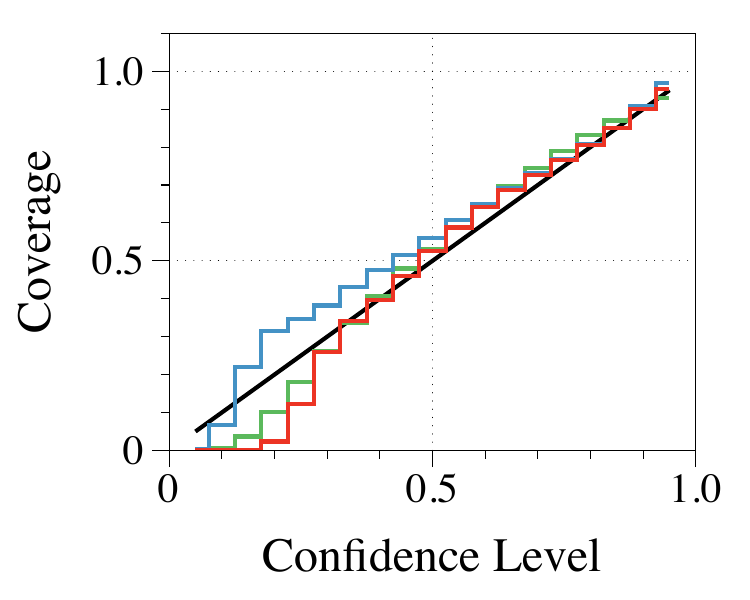}%
}%
\caption{Coverage of different confidence levels on four datasets.
}
\label{coverage}
\end{figure}
\begin{figure}[thb!]
\centering
\subfigure[StackOverflow]{   
    \includegraphics[width=0.24\linewidth]{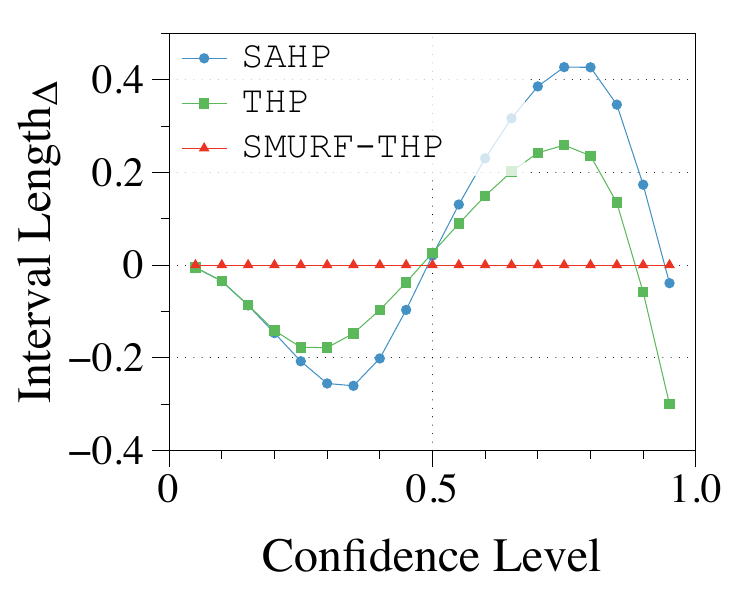}%
}%
\subfigure[Retweet]{   
    \includegraphics[width=0.24\linewidth]{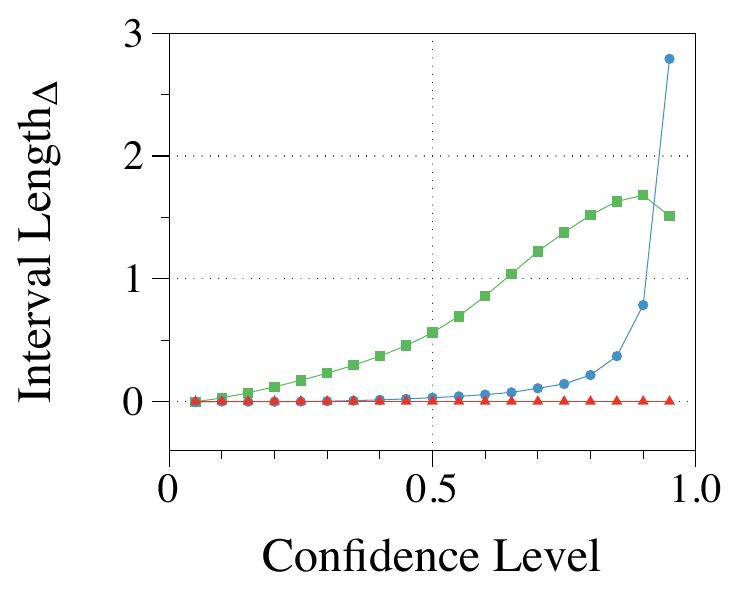}%
}%
\subfigure[MIMIC-II]{   
    \includegraphics[width=0.24\linewidth]{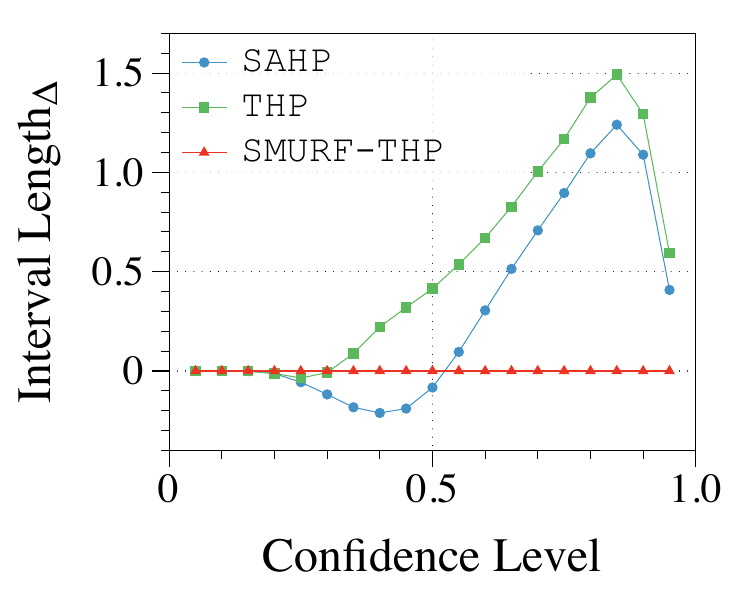}%
}%
\subfigure[Financial]{   
    \includegraphics[width=0.24\linewidth]{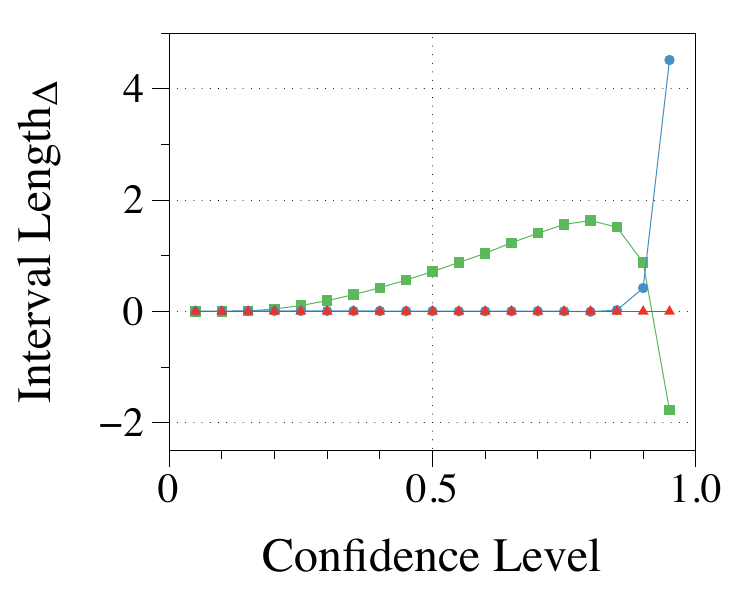}%
}%
\caption{Interval Length of different confidence levels on four datasets. We minus the interval length of all methods with the length of $\smurf$ (denoted as Interval Length$_\Delta$) to better show the differences.
}
\label{intlen}
\end{figure}

\textbf{Distribution of Samples.}
We visualize the distribution of samples generated by $\smurf$ for several events to study the predicted intensity and present them in Figure~\ref{dis}.
A large proportion of the samples stay close to value 0, which is reasonable since most of the events occur within a short time after the previous one.
Distributions vary as time goes further. Figure~\ref{dis_3} and Figure~\ref{dis_4} exhibit that \smurf can still present tiny peaks around the ground truth value,
indicating that our model can still capture the effect of historical events.
Yet, the generated samples may be inaccurate due to the huge randomness around the large values of arrival time.

\begin{figure}[thb!]
\centering
\subfigure[Ground truth $t \in [0,1)$]{   
    \label{dis_1}
    \includegraphics[width=0.24\linewidth]{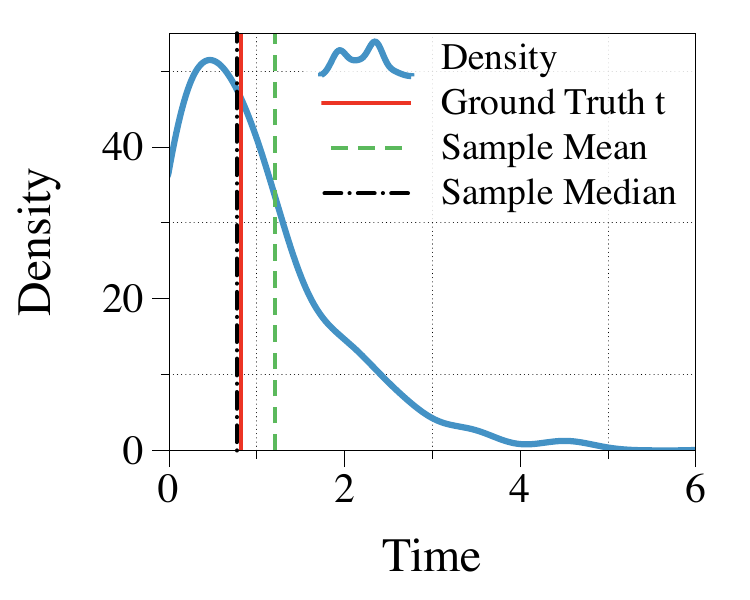}%
}%
\subfigure[Ground truth $t \in [1,2)$]{   
    \label{dis_2}
    \includegraphics[width=0.24\linewidth]{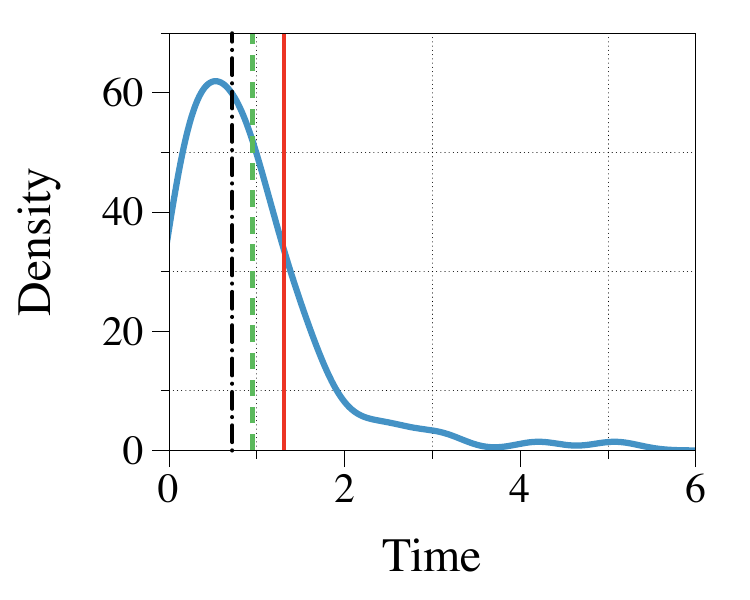}%
}%
\subfigure[Ground truth $t \in [4,5)$]{   
    \label{dis_3}
    \includegraphics[width=0.24\linewidth]{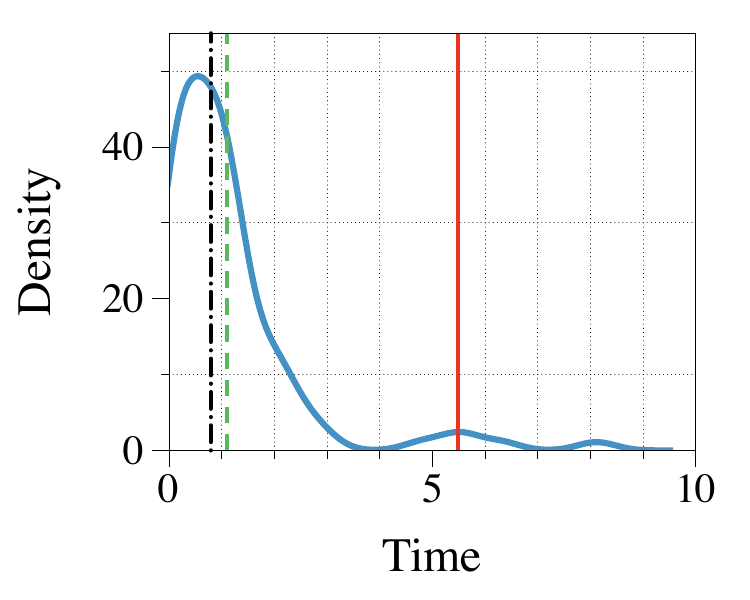}%
}%
\subfigure[Ground truth $t \in [5,6)$]{   
    \label{dis_4}
    \includegraphics[width=0.24\linewidth]{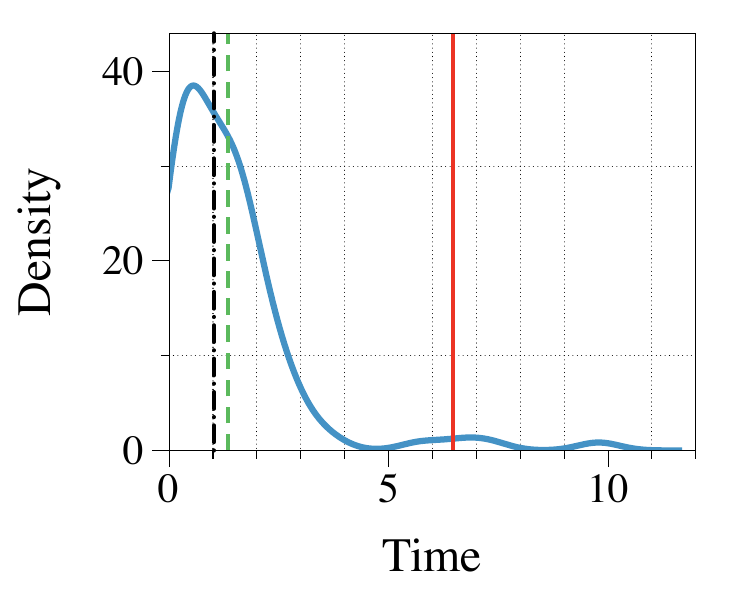}%
}%
\caption{Distribution of samples generated by $\smurf$ and the ground truth value $t$ in the StackOverflow dataset. The green and black lines signify the mean and median values of the samples. Ground truth of the event arrival time $t$ is chosen randomly from time intervals $[0,1), [1,2), [4,5)$, and $[5,6)$.}
\label{dis}
\end{figure}

\subsection{Ablation Study}
\textbf{Parametrization.}
In addition to parametrizing the intensity function, we can also parametrize the score function directly as was done in conventional score-based methods.
That is, we can mask the intensity function and directly parametrize the corresponding score function with the neural network.
Table~\ref{parame} summarizes the results on the StackOverflow and Retweet datasets, where $\smurf^\texttt{s}$ indicates $\smurf$ with the score function parametrized.
Results imply that parametrizing the intensity function fits the sequence better than parametrizing the score function.
This might occur because deriving the score function from the intensity implicitly narrows the searching space, which facilitates the model's convergence.
\begin{table}[thb!]
\centering
\caption{Calibration Score and CRPS of \smurf with two different parametrizations. The superscript \texttt{s} indicates a direct parametrization on the score function as opposed to no superscript indicating a parameterization on the intensity function.}
\vskip 0.15in
\label{parame}
\begin{tabular}{lcccc}
\toprule
                                                   & \multicolumn{2}{c}{StackOverflow}   & \multicolumn{2}{c}{Retweet}         \\ \cmidrule(r){2-3} \cmidrule(r){4-5} 
Methods                                            & $\mathrm{CS}(\%)$ & $\mathrm{CRPS}$ & $\mathrm{CS}(\%)$ & $\mathrm{CRPS}$ \\
\midrule
$\smurf^{\texttt{s}}$ & 6.13             & 0.54            & 3.92             & 1.33            \\
$\smurf$                  & \textbf{0.65}              & \textbf{0.44}            & \textbf{0.71}              & \textbf{0.86}            \\ 
\bottomrule
\end{tabular}%
\end{table}

\textbf{Denoising.}
\smurf adds Gaussian noise to data points and adopts denoising score matching for better performance and computational efficiency.
We test the effects of different noise scales $\sigma$ in Figure~\ref{noise}.
Compared with training on clean data, adding perturbations effectively improves the performance with a suitable noise scale.
Nevertheless, the selection of noise scale is tricky.
Adding too small noise will lead to unstable training, which does not cover the low-density regions enough and degrades the performance.
A too large noise, on the other hand, over-corrupts the data and alters the distribution significantly from the original one.
We simply select $\sigma$ by grid search in the above experiments and leave auto-selection for future work.

\textbf{Sampling Method.}
We inspect variants of our sampling algorithm, including Langevin Dynamics without the denoising step and Mirror Langevin Dynamics, on two datasets.
Table~\ref{sampling_abla} summarizes the results.
Langevin with the denoising step yields the best performance since the denoising step corrects the perturbed distribution and ensures non-negativity.
Mirror Langevin yields the worst results, likely due to the fact that \smurf learns the score of a perturbed distribution which is not restricted to a constrained region.

\begin{table}[thbp]
\centering
\caption{Calibration Score and CRPS of $\smurf$ with different sampling algorithms on StackOverflow and Retweet datasets.}
\vskip 0.15in
\label{sampling_abla}
\begin{tabular}{lcccc}
\toprule
                   & \multicolumn{2}{c}{StackOverflow}                                & \multicolumn{2}{c}{Retweet}                           \\ \cmidrule(r){2-3} \cmidrule(r){4-5} 
Sampling Algorithm & $\mathrm{CS}(\%)$ & $\mathrm{CRPS}$ & $\mathrm{CS}(\%)$ & $\mathrm{CRPS}$ \\
\midrule
Langevin+DS        & \textbf{0.65}                       & \textbf{0.44}   & \textbf{0.71}                       & \textbf{0.86}   \\
Langevin           & 0.87                                & 0.46            & 1.38                                & 0.92            \\
Mirror Langevin    & 8.88                               & 0.69            & 6.73                               & 1.29            \\ \bottomrule
\end{tabular}%
\end{table}

\textbf{Accuracy with ground truth timestamp.}
We measure the type prediction accuracy of \smurf given ground truth event time as the inputs of the classifier $g$.
As presented in Table~\ref{acc_gt}, the prediction accuracy of $\mathrm{\smurf}$ exhibits an obvious increase compared with results in Table~\ref{main_results}.
This indicates the value of arrival time for type prediction and also demonstrates that \smurf successfully captures the time-type dependencies.
\label{gt-acc-sec}
\begin{table}[thb!]
\centering
\caption{Event type prediction accuracy of $\smurf$ given ground truth event time as the inputs.}
\vskip 0.15in
\label{acc_gt}
\begin{tabular}{cc}
\toprule
              & $\mathrm{\smurf}$ \\ \cmidrule{2-2} 
Dataset       & Acc(\%)                       \\ \midrule
StackOverflow & 48.38                         \\
Retweet       & 61.65                         \\
MIMIC-II         & 83.84                         \\
Financial     & 61.62                         \\ \bottomrule
\end{tabular}%
\vspace{-2ex}
\end{table}

\textbf{Hyperparameter Sensitivity.}
We examine the sensitivity of the hyperparameter $\alpha$ in the training objectives to our model's prediction performance on the StackOverflow dataset. 
As shown in Figure~\ref{training_weight}, the Calibration Score and CRPS become lower as $\alpha$ grows, while the accuracy worsens.
This accords with intuition since a larger $\alpha$ indicates more attention on fitting event time and less on predicting event type.
In the above experiments, we choose the biggest $\alpha$ that can achieve comparable accuracy with THP.

We also train \smurf with different volumes of training data to study its generalization ability.
We train the model on different ratios of the dataset and present the performance in Figure~\ref{train_size_1} and Figure~\ref{train_size_2}.
As shown, all metrics go better as we feed more training data. 
Compared to the Retweet dataset, \smurf is more sensitive to training ratio on the StackOverflow dataset. This is due to that the StackOverflow dataset contains less events than the Retweet, thereby \smurf requires larger proportion to learn the distribution.

\begin{figure}[thbp]
\centering
\subfigure[Retweet]{   
    \label{noise_retweet}
    \includegraphics[width=0.30\linewidth]{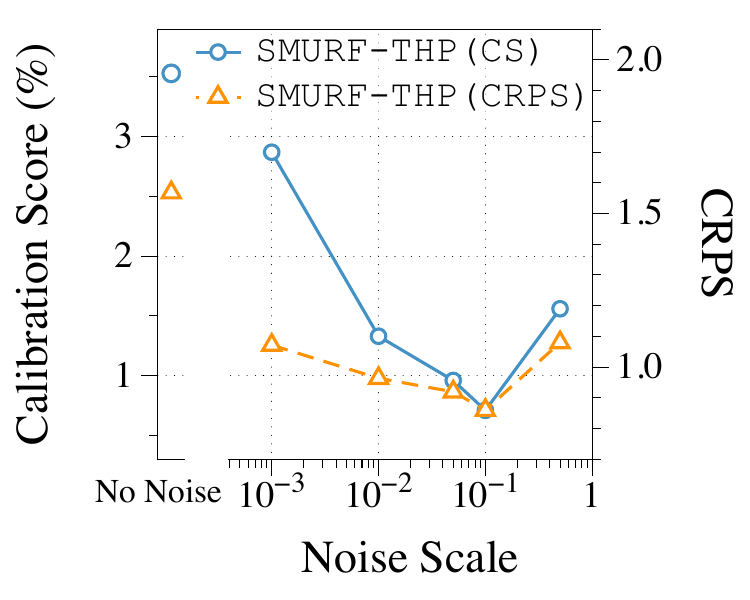}%
}%
\subfigure[StackOverflow]{   
    \label{noise_so}
    \includegraphics[width=0.30\linewidth]{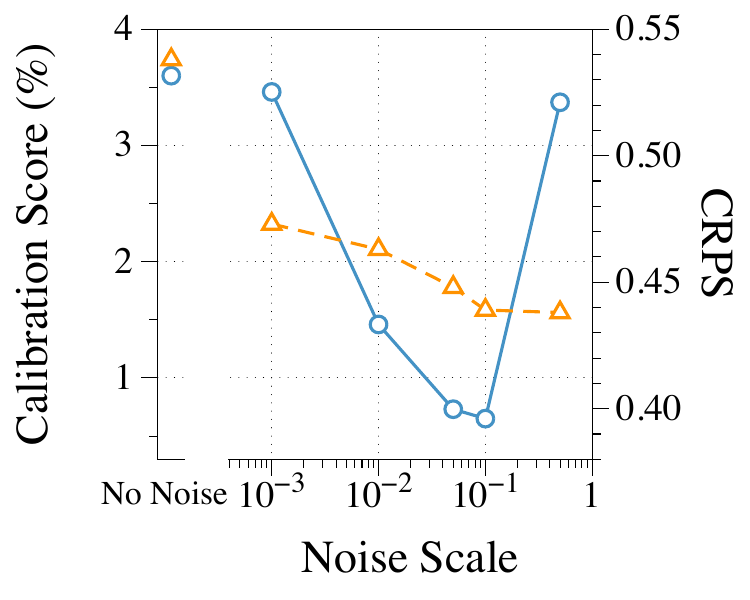}%
}%
\caption{Comparison of Calibration Score and CRPS of $\smurf$ with different noise scales and with no noise added.}
\label{noise}
\end{figure}

\begin{figure}[thb!]
\centering
\subfigure[Calibration Score and CRPS]{   
    \label{train_weight_cs}
    \includegraphics[width=0.30\linewidth]{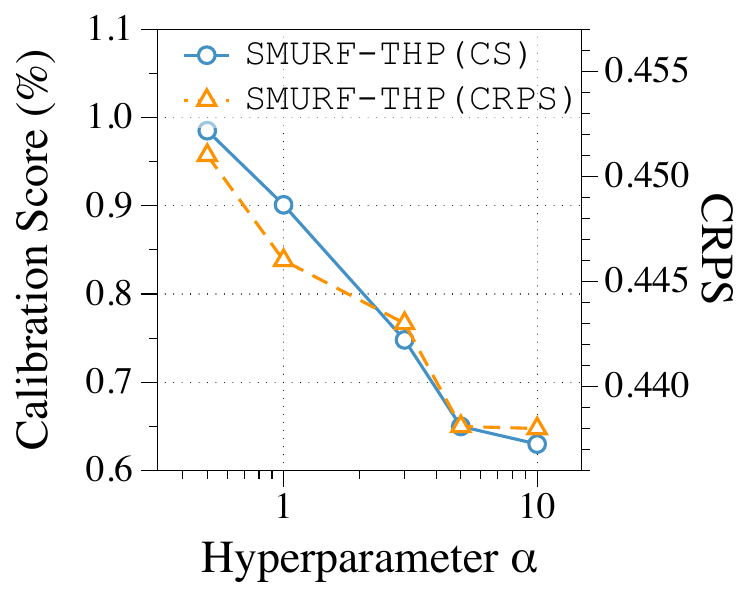}%
}%
\subfigure[Accuracy]{   
    \label{train_weight_acc}
    \includegraphics[width=0.30\linewidth]{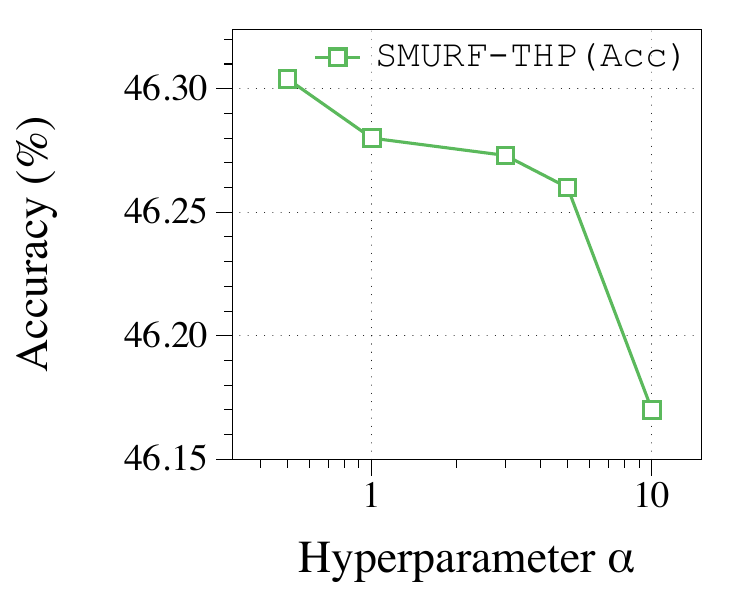}%
}%
\caption[]{Sensitivity of the hyperparameter $\alpha$ in the training objective to prediction performance.}
\label{training_weight}
\end{figure}

\begin{figure}[thb!]
\centering
\subfigure[StackOverflow]{   
    \includegraphics[width=0.30\linewidth]{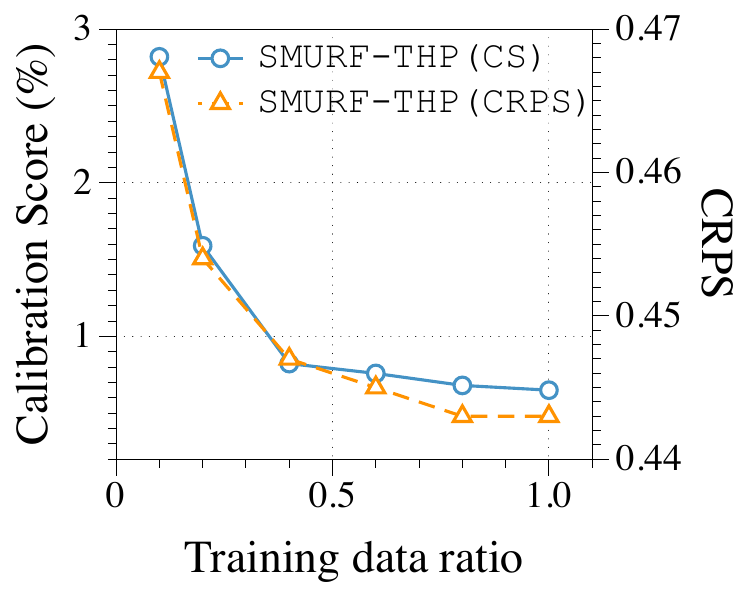}%
    }
\subfigure[Retweet]{  
    \includegraphics[width=0.30\linewidth]{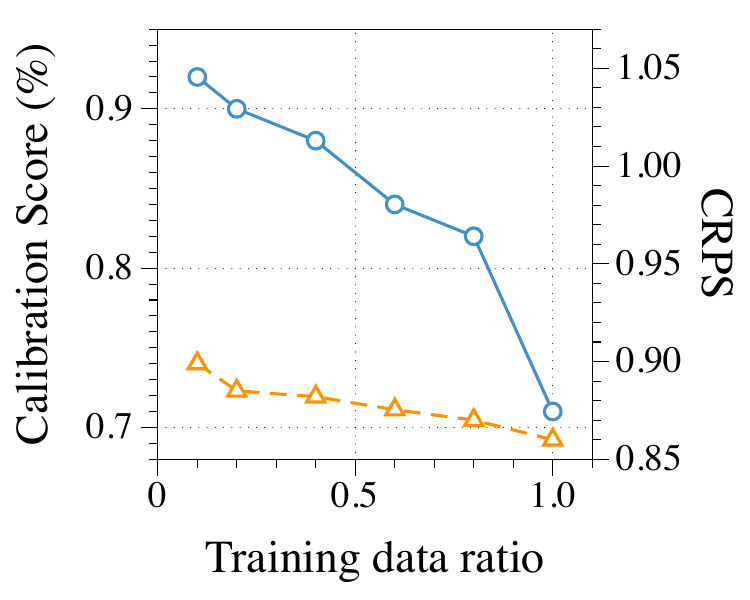}%
}%
\caption[]{Calibration Score and CRPS of $\smurf$ trained with different ratios of the StackOverflow and Retweet datasets.}
\label{train_size_1}
\vspace{-2.5ex}
\end{figure}

\begin{figure}[thb!]
\centering
\includegraphics[height=0.30\linewidth]{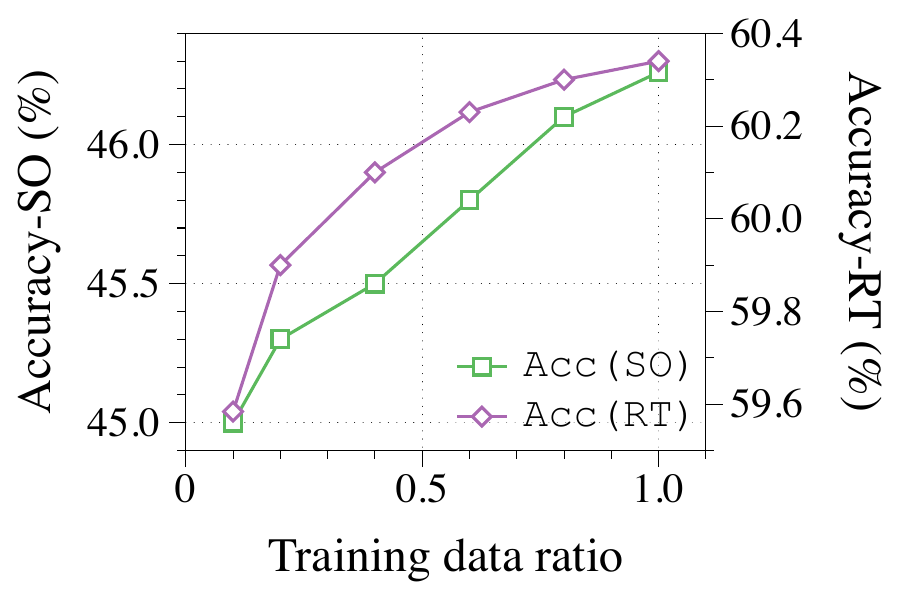}%
\caption{Prediction Accuracy of $\smurf$ trained with different ratio of the StackOverflow and Retweet datasets.}
\label{train_size_2}
\vspace{-2.5ex}
\end{figure}

\section{Discussion and Conclusion}
We acknowledge the existence of several studies that adopt non-likelihood-based estimators to circumvent the intractable integral within the log-likelihood computation. We present our discussions on these works below.
\citet{DIS2017} trains the model to directly predict the next event's time and type through a summation of Mean Squared Error (MSE) and cross-entropy loss.
However, their model does not construct an explicit intensity function and hence doesn’t support flexible sampling and uncertainty quantification.
TPPRL \citep{TPPRL2018} employs reinforcement learning (RL) for learning an events generation policy, but they concentrate on the temporal point process (TPP) rather than the marked temporal point process (MTPP), which is the focus of our work. 
Furthermore, they presume the intensity function to be constant between timestamps, a limitation that hampers the accurate capture of the point process's temporal dynamics.
\citet{RLPP2018} applies RL for MTPP which trains a policy to maximize feedback from the environment.
Similar to TPPRL, it assumes a very stringent intensity function, e.g., in exponential forms, which is oversimplified to capture the complex point process in real-world applications.
INITIATOR \citep{initiator2018} and NCE-TPP \citep{ncetpp2020} are both based on noise-contrastive estimations for MTPP. However, they utilize the likelihood objective for training the noise generation network, which consequently reintroduces the intractable integral. In our experiments, we include NCE-TPP's performance, as its authors have demonstrated it outperforms INITIATOR.

Several other works \citep{uncertain-infer, uncertain-spatial} explore different scopes of uncertainty quantification for the Hawkes process. That is, they provide uncertainty quantification for the parameters in conventional Hawkes process models, whereas we focus on uncertainty quantification for the predicted arrival time.

In this work, we present \smurf, a score-based method for training Transformer Hawkes process models and quantifying prediction uncertainty.
Our proposed model adopts score matching as the training objective to avoid intractable computations in conventional Hawkes process models.
Moreover, with the learnt score function, we can sample arrival time of events using the Langevin Dynamics. This enables uncertainty quantification by calculating the associated confidence interval.
Experiments on various real-world datasets demonstrate that \smurf achieves state-of-the-art performance in terms of Calibration Score, CRPS and Interval Length.

\clearpage
\bibliography{main}
\bibliographystyle{apalike}

\newpage
\appendix
\onecolumn

\section{Training Detail}
\label{app:train}
\begin{table}[h]
\centering
\caption{Summary of hyperparameters used in training.}
\label{hyper}
\resizebox{\textwidth}{!}{%
\begin{tabular}{cccccccccccc}
\hline
Dataset       & \#head & \#layer & $d_\mathrm{model}$ & $d_k=d_v$ & $d_{\mathrm{hidden}}$ & dropout & batch & lr   & noise scale & Langevin step size $\epsilon$ & \#step \\ \hline
StackOverflow & 4     & 4      & 64          & 16        & 256          & 0.1     & 4     & 1e-4 & 1e-1        & 5e-2                          & 1000  \\
Retweet       & 3     & 3      & 64          & 16        & 256          & 0.1     & 16    & 5e-3 & 5e-2        & 1e-3                          & 5000  \\
MIMIC-II      & 3     & 3      & 64          & 16        & 256          & 0.1     & 1     & 1e-4 & 1e-1        & 2e-2                          & 1000  \\
Financial     & 6     & 6      & 128         & 64        & 2048         & 0.1     & 1     & 1e-4 & 5e-2        & 5e-3                          & 3000  \\ \hline
\end{tabular}%
}
\end{table}

For all datasets used in the experiments, we employ the same data pre-processing and train-dev-test split as \citet{rmtpp2016} and \citet{nhp2017}.
Please found more details and downloadable links in the aforementioned papers.
In addition, we apply normalization to the Retweet, MIMIC-II and MemeTrack datasets to keep the time scale consistent.
The hyper-parameters we use are summarized in Table~\ref{hyper}.
We adopt the Adam optimizer and train all models with 50 epochs using an NVIDIA Tesla P100 GPU.
The number of perturbation and samples are set to be 100, increasing which may improve model's performance.

\end{document}